%% file: iclr2027_arxiv.tex
\documentclass{article} 
\usepackage{iclr2027_conference,times}

\newif\ifpreprint
\preprinttrue  

\ifpreprint
    \iclrfinalcopy
\fi

\input{math_commands.tex}

\usepackage{hyperref}
\usepackage{url}
\usepackage{enumitem}
\usepackage{listings}
\usepackage{amssymb}
\usepackage{booktabs} 
\usepackage{multirow}
\usepackage{tikz}
\usepackage{graphicx}
\usepackage{amsmath}
\usepackage{amsthm}
\usepackage{algorithm}
\usepackage{algpseudocode}
\usepackage{atnguyen_iclr}

\definecolor{ablatedblue}{HTML}{287FA8}
\definecolor{intactgray}{HTML}{858585}

\title{Can Capability Enhancement be Localized?}

\author{
    Anh Tuan Nguyen, Anna Bair, Zico Kolter \\
    { Carnegie Mellon University} \\ 
    { \texttt{\{atnguyen,abair,zkolter\}@cs.cmu.edu}}
}

\begin{document}

\maketitle

\ifpreprint
    \fancyhead[L]{Preprint}
\fi

\begin{abstract}

    Transformer-based language models have shown strong capabilities across many tasks and domains. However, prior work shows that, for a given task, not all attention heads matter equally. Specifically, they suggest that many \textit{redundant} heads can be safely removed with little-to-no effect, while a small set of \textit{important} heads, when ablated, severely damages the model's performance on the given task. In this work, we investigate a complementary aspect, called \textit{Capability Enhancement Localization} (CEL): \textit{given a model and a task, does there exist a small set of heads whose removal improves the task performance?} We first investigate this question \textit{empirically} across models from different families. Using a simple two-stage routine that screens candidate heads and evaluates their joint ablations, we identify sparse head combinations whose removal improves task-specific performance without further fine-tuning. Motivated by these observations, we ask \textit{theoretically whether such a phenomenon can naturally arise from training dynamics, and under what conditions}. To this end, we construct a simplified multi-head softmax attention model and two tasks with conflicting features, and show that gradient descent from small random initialization provably concentrates the conflicting features into particular heads. We further show that ablating such heads improves performance on one task while harming the other, even though training converges to a global optimum of the joint objective. This establishes a possible training-induced theoretical explanation for CEL.
\end{abstract}
\section{Introduction}

Large language models (LLMs) \citep{vaswani2017attention, hoffmann2022training} have demonstrated strong capabilities across diverse tasks, including language understanding, text generation, question answering, code generation, and mathematical reasoning \citep{brown2020language, raffel2020exploring, rae2021scaling}, among others. A notable feature of LLMs is that a single set of parameters supports many different tasks, which can often be specified through natural-language instructions or a few demonstrations without task-specific parameter updates \citep{brown2020language, wei2021finetuned}. This breadth of functionality motivates a line of work studying how a shared network supports diverse capabilities  and how its internal components (e.g., attention heads/layers) contribute to LLM performance on individual tasks. 

Early works on LLM pruning demonstrated that the attention heads exhibit substantial redundancy and contribute unequally across different tasks. For example, \citet{clark2019does} observed that attention heads within the same BERT layers \citep{devlin2019bert} often exhibit similar attention patterns, leading to overlapping behavior. Moreover, \citet{michel2019sixteen} and \citet{voita2019analyzing} showed that many attention heads can be pruned after training without substantially affecting the overall LLM performance. Together, those findings emphasize that the full complement of attention heads is generally not required at inference time for preserving task capabilities, motivating a closer look at which heads are dispensable and which are critical for particular tasks. 

Besides the evidence of redundancy, recent works have identified attention heads with specialized roles. Examples include induction mechanism corresponding to in-context learning \citep{olsson2022context}, information retrieval from long contexts \citep{wu2025retrieval}, and processing operands and operators for arithmetic tasks \citep{zhang2024interpreting}, to name a few. Other work identified distinct heads that favor internal knowledge encoded in the model's weights or external context when these information sources conflict \citep{jin2024cutting}. These analyses revealed how specific heads support or route computations under various scenarios. More recently, broader capability-localization studies complement these findings. In particular, \citet{huang2025capability} identified a small group of \textit{neurons} associated with shared capabilities, while \citet{bair2026compressed} leveraged a cheap compressed sensing method (i.e., LASSO \citep{tibshirani1996regression}) to identify \textit{task-critical attention heads} corresponding to separate capabilities. Ablating just a few of these task-critical heads can \textit{dramatically degrade} the LLM performance on the corresponding task.

The above findings provided evidence of \textit{redundant} and \textit{task-critical} attention heads. In this work, we complement these findings by asking the following counter-intuitive question: \textit{given an arbitrary task and an LLM, is there a sparse set of attention heads such that ablating them improves the performance of that LLM on that task, without any further fine-tuning?} For convenience, we term this question \textit{Capability Enhancement Localization} (CEL). 

We note that beneficial head removal has already been observed in some \textit{specific} scenarios. \citet{held2023shapley} identify heads that interfere with performance in multilingual models and show that pruning them can improve performance for a target language. \citet{jin2024cutting} prune heads involved in conflicts between a model's internal knowledge and external context. More recently, \citet{nguyen2025structured} observe that pruning certain heads can improve mathematical reasoning and use different pruned configurations to generate diverse answers for best-of-$N$ selection. In this work, we ask whether sparse beneficial head ablations can be found across a \textit{broader range of tasks and models}, without presupposing a particular source of interference. More importantly, we ask \textit{how successful training can produce a head whose removal improves a task}. We answer this theoretical question in a simplified attention model, where task conflict and learned feature allocation give rise to beneficial head ablation. 

\textbf{Contribution.}  In this work, we study the \textit{Capability Enhancement Localization} (CEL) question, which asks whether there exists a sparse set of attention heads whose ablation improves task performance without further fine-tuning, given a task. Our contributions are:
\begin{enumerate}[leftmargin=*]
    \item In Section \ref{sec:empirical-result}, we consider multiple model families and task configurations. Using a simple two-stage approach, we found evidence of the CEL phenomenon: given a task, ablating a sparse set of heads can significantly improve models' task-specific performance.
    \item Motivated by empirical evidence of CEL, we seek a potential theoretical explanation for whether and how this phenomenon naturally arises in LLM training. To this end, in Section~\ref{sec:theoretical-explanation}, we construct a simplified softmax attention model, trained on two tasks with a feature conflict. We then show in Theorem~\ref{thm:training-induce-localization} that, under standard GD training, such conflicting features can be encoded into a specific attention head, leading to the \textit{localization} phenomenon.
    \item We then further show in Theorem~\ref{thm:training-induce-ablation} that ablating such a head can improve the model's performance on a single task while harming the other, even though training converges to an optimum of the joint objective. Together with the localization phenomenon, this theoretically establishes one possible mechanism for CEL.  
\end{enumerate}

\section{Related Work}

\textbf{Attention-head functionality.}
Studies of Transformer components reveal both redundancy and functional specialization. Many attention heads can be removed with little loss in task performance \citep{michel2019sixteen, voita2019analyzing}, while others participate in identifiable computations, including induction and information retrieval \citep{olsson2022context, wu2025retrieval}. Beyond individual head functions, localization studies investigate components associated with factual associations, task-specific skills, and capabilities shared across examples \citep{meng2022locating, panigrahi2023task, huang2025capability}. At the attention-head level, \citet{bair2026compressed} identify sparse sets of task-critical heads whose joint ablation substantially degrades target-task performance while largely preserving performance on unrelated tasks. Together, these findings motivate examining how sparse interventions affect task performance, including whether their effects can be beneficial. In this work, \textit{we instead focus on the enhancement regime}: identifying small sets of attention heads whose joint removal improves task performance without further fine-tuning. Furthermore, we also seek a \textit{theoretical explanation} for how this phenomenon can arise in standard LLM training.

\textbf{Performance enhancement through head pruning.}
Prior work has found benefits from head removal in several settings. \citet{held2023shapley} show that pruning heads associated with interference across languages can improve performance for a target language. \citet{jin2024cutting} study conflicts between internal knowledge and external context, and show that pruning selected heads can improve question answering. In mathematical reasoning, \citet{nguyen2025structured} observe improvements from pruning particular heads and use different pruned configurations to produce diverse answers for best-of-$N$ selection. In this work,  we consider a more general setting by examining sparse beneficial interventions across a broader range of tasks and models, and provide a potential theoretical explanation of how this phenomenon can arise in LLM training dynamics, an aspect that was ignored by prior works.

\textbf{Mechanistic interpretability and circuit discovery.}
Mechanistic interpretability studies how internal components implement model behavior. \citet{wang2022interpretability} identify an attention-head circuit for indirect object identification, while \citet{conmy2023towards} use activation patching to automate the discovery of computational subgraphs. Such localization need not predict the effect of changing a model: \citet{hase2023does} find that causal-tracing results do not identify the best MLP layers for factual editing. In our experiments, head ablation measures the effect of a specified intervention but does not reveal the computation performed by the removed heads. Our theory instead shows, in a simplified model, how training can produce a head whose removal benefits one task.

\section{Empirical Motivation: Sparse Beneficial Head Ablation} \label{sec:empirical-result}

In this section, we introduce our notation, define task-specific performance gain under head ablations, and formulate the \textit{Capability Enhancement Localization } (CEL) problem as discovering small head subsets whose removal improves task-specific performance of the considered model.

\subsection{Head ablation and evaluation}

\textbf{Notation.} We consider a transformer-based LLM $M$ with $L$ layers and $H$ attention heads per layer, with $N = LH$ heads in total. We denote the head at layer $\ell \in \{1, \dots, L\}$ and position $i \in \{1, \dots, H\}$ by $h_{\ell, i}$, and the set of all attention heads by $\mathcal{H} = \left\{h_{\ell,i} : \ell \in \{1,\ldots,L \}, i \in \{1,\ldots,H\} \right\}.$ Our intervention targets a \textit{sparse} subset of attention heads $S \subsetneq \mathcal{H}$, with the number of elements $|S| \leq k$ for $k \ll N$. In this work, we typically choose $k = 3$.
 
\textbf{Head ablation, performance gain, and beneficial head subset.} Given a set $S \subsetneq \mathcal{H}$ of attention heads, let $M_{\setminus S}$ denote the model obtained by setting the output of heads in $S$ to zero at every token position before the attention output projection. We note that \textit{all remaining parameters remain unchanged, and no fine-tuning is performed}. For a task $T$ with a \textit{test} dataset $D_\textup{test}$, we denote the performance of the model $M$ under the task $T$ as 
\[
    \textup{Acc}_{T, D_{\textup{test}}}(M) = \frac{1}{|D_\textup{test}|}\sum_{(x, y) \in D_\textup{test}} \mathbf{1}\!\left[
    M \text{ answers } x \text{ correctly under the task scorer}
  \right]
\]
Under the same inference and scoring protocol for the intact and ablated models, we denote the \textit{empirical ablation gain} as $\Delta_{T, D_\textup{test}}(S) = \textup{Acc}_{T, D_\textbf{test}}(M_{\setminus S}) - \operatorname{Acc}_{T, D_\textup{test}}(M).$
We then call the sparse subset of attention heads $S$ \textit{beneficial} on the test set $D_\textup{test}$ if $\Delta_{T, D_\textup{test}}(S) > 0$. 

\textbf{Objective.} Given a model $M$, a task $T$ with a test set $D_\textup{test}$, and a head budget $k$, we use a \textit{disjoint discovery set} $D_\textup{dist}$ (i.e., $D_\textup{dist} \cap D_\textup{test} = \emptyset$) from the task $T$ to select a head subset through a search routine $\mathcal{A}$:
\[
  \widehat{S} = \mathcal{A}(M,T,D_\textup{dist},k),
  \qquad 1 \leq |\widehat{S}| \leq k.
\]
The goal of CEL is to identify subsets $\hat{S}$ whose joint ablation yields a large positive test gain $\Delta_{T,D_\textup{test}}(\widehat{S})$, while using only discovery data $D_\textup{dist}$ to guide selection. The selected subset $\hat{S}$ is fixed before test evaluation. By using this setup, our hope is that the identified subset $\hat{S}$ truly improves the task-specific performance on task $T$, rather than only overfitting the discovery set $D_\textup{dist}$.

\subsection{Experimental Setup}
\textbf{Models and tasks.} We evaluate three models of different families: Llama-3.2-3B-Instruct, Qwen2.5-3B-Instruct, and SmolLM3-3B-Instruct \citep{grattafiori2024llama, qwen2025qwen2, bakouch2025smollm3}. The number of layers and heads of considered models can be found in Table~\ref{tab:model_details}. For the tasks, we consider three capabilities using various datasets: (1) math capabilities with Arithmetic \citep{brown2020language}, (2) code generation with MBPP \citep{austin2021program}, and (3) language skills via HellaSwag \citep{zellers2019hellaswag} and BoolQ \citep{clark2019boolq}.

\textbf{Intervention selection.} For each model--task pair, we use a two-stage procedure on the discovery set $D_\textup{dist}$. We first score individual heads to form a small candidate pool, then evaluate joint ablations of subsets drawn from that pool and select one subset with the best discovery-set performance. We fix this subset before measuring its gain on the held-out test set $D_\textup{test}$. The test examples play no role in scoring heads or choosing the intervention. Details of the procedure are given in Algorithm~\ref{alg:head-subset-discovery}. For the dataset MBPP, we use the official test set as $D_{test}$ and subsample 374 examples for the training set as $D_\textup{dist}$. For the datasets HellaSwag and BoolQ, we use the official validation datasets as the test set, and subsample 2,000 examples from the official training sets as discovery set $D_{dist}$. For the dataset Arithmetic, since it does not have a separate test set with official labels, we select from its official validation split of 20,000 examples 2,000 examples for $D_\textup{dist}$ and another disjoint 1,000 examples as $D_\textup{test}$. For convenience, we summarize the number of examples of discovery and test sets for each task in Table~\ref{tab:dataset_sizes}.

\textbf{Hardware configuration.} We ran the experiments as Slurm jobs, each allocated one NVIDIA L40S GPU, two CPUs, 24 GB of system memory, and with a batch size of 8.

\begin{algorithm}[t]
  \caption{Two-stage intervention selection and held-out evaluation}
  \label{alg:head-subset-discovery}
  \begin{algorithmic}[1]
    \Require Model $M$ with head set $\mathcal H$; task $T$;
             discovery set $D_{\textup{dist}}$; test set $D_{\textup{test}}$
    \Require Candidate count $p=10$; head budget $k=3$
    \Ensure Selected subset $\widehat S$ and its test-set gain

    \State For $D\in\{D_{\textup{dist}},D_{\textup{test}}\}$, use
           $\operatorname{Acc}_{T,D}(M_{\setminus S})$ and
           $\Delta_{T,D}(S)$ as defined above
    \State Let $d_{T,D}(S)$ count examples in $D$ answered correctly by
           $M$ but incorrectly by $M_{\setminus S}$
    \State Rank subsets on $D_{\textup{dist}}$ by higher gain, lower damage,
           higher accuracy, fewer heads, then lexicographic head order

    \Statex \textit{Stage 1: Screen individual heads}
    \For{$h\in\mathcal H$}
      \State Evaluate $\{h\}$ on $D_{\textup{dist}}$
    \EndFor
    \State $P\gets$ heads in the top $p$ singleton interventions

    \Statex \textit{Stage 2: Select a joint intervention}
    \State $\mathcal C\gets\{S\subseteq P:1\leq |S|\leq k\}$
    \For{$S\in\mathcal C$ with $|S|\geq 2$}
      \State Evaluate $S$ on $D_{\textup{dist}}$
            \Comment{Reuse singleton results}
    \EndFor
    \State $\widehat S\gets$ highest-ranked subset in $\mathcal C$
           on $D_{\textup{dist}}$
    \State Freeze $\widehat S$

    \Statex \textit{Held-out evaluation}
    \State Evaluate $\widehat S$ on $D_{\textup{test}}$ with a matched
           intact-model run
    \State \Return $\widehat S$,
           $\operatorname{Acc}_{T,D_{\textup{test}}}(M_{\setminus\widehat S})$,
           $\Delta_{T,D_{\textup{test}}}(\widehat S)$
  \end{algorithmic}
\end{algorithm}

\subsection{Results}

Table~\ref{tab:heldout-ablation-results} reports the held-out accuracy changes for the model--task pairs studied here. Each gain compares an ablated model with an intact model evaluated under the same task protocol. The reported interventions improve held-out accuracy across Arithmetic, BoolQ, HellaSwag, and MBPP. For example, the SmolLM3 Arithmetic intervention raises accuracy from $70.40\%$ to $83.50\%$, a gain of $13.10$ percentage points. The gains vary across model--task pairs, so the evidence supports beneficial sparse ablations in these tested settings, rather than a claim that ablation will help every model on every task.

Our two-stage search is \textit{deliberately simple}. It searches small subsets within a pool selected using individual-head effects, so it need not find the best subset among all heads. These results establish the empirical phenomenon but do not explain why training can produce components whose removal improves a particular task. We next study this question in a simplified attention model. The theory provides one possible mechanism for beneficial head ablation; it does not identify the mechanism responsible for the gains in the language models above.


\begin{table*}[t]
    \centering
    \caption{Held-out performance of head ablations. SmolLM3 is evaluated in non-thinking mode. }
    \label{tab:heldout-ablation-results}
    \small
    \begin{tabular}{llcccc}
    \toprule
    Model & Task & Intact acc. & Ablated acc. & Gain (pp) & Rel. error reduction \\
    \midrule
    SmolLM3 & Arithmetic & 70.40\% & 83.50\% & +13.10 & 44.26\% \\
             & BoolQ      & 82.57\% & 84.92\% & +2.35  & 13.51\% \\
             & HellaSwag  & 75.75\% & 77.05\% & +1.29  & 5.34\% \\
             & MBPP       & 50.40\% & 52.00\% & +1.60  & 3.23\% \\
    \midrule
    Qwen2.5-3B-Instruct & Arithmetic & 89.50\% & 90.90\% & +1.40 & 13.33\% \\
                        & BoolQ      & 80.21 & 82.35 & +2.14 & 10.82\% \\
                        & HellaSwag  & 75.02 & 76.66 & +1.63 & 6.54\% \\
                        & MBPP       & 60.40\% & 62.20 & +1.80 & 4.55\% \\
    \midrule
    Llama-3.2-3B-Instruct & Arithmetic & 68.00\% & 74.30\% & +6.30 & 19.69\% \\
                          & BoolQ      & 78.72\% & 81.25\% & +2.54 & 11.93\% \\
                          & HellaSwag  & 71.64\% & 73.78\% & +2.14 & 7.55\% \\
                          & MBPP       & 47.80\ & 53.00\% & +5.20 & 9.96\% \\
    \bottomrule
    \end{tabular}
    \vspace{2pt}
    \begin{minipage}{0.98\textwidth}
    \end{minipage}
\end{table*}

\section{A Theoretical Explanation of the Capability Enhancement Localization Phenomenon}  \label{sec:theoretical-explanation}
As we saw previously, there is strong evidence suggesting the existence of the CEL phenomenon. But why does it happen? In this section, we will investigate this question via a simple model with \textit{diagonal weights} and \textit{small initialization scale}, inspired by prior works \citep{boix2023transformers, pesme2023saddle, berthier2023incremental}. We then show that the CEL phenomenon can potentially arise naturally in the training trajectory if there is initially some conflict between tasks.

\subsection{Tasks, Attention Model, and Training} \label{sec:attention-model-setting}
In this section, we will first describe our simple model within which we study the CEL problem.

\paragraph{Data and task.} Let $U, V, W$ be independent Rademacher variables which take the value in $\{-1, 1\}$ with equal chance. Fix real values $a, b, c > 0$, and set $t = \frac{a}{2}$. We define two equally weighted training tasks
\[
    A:\quad X=(U,V,0),\quad Y=bV,
    \quad \textup{and } \quad
    B:\quad X=(U,0,W),\quad Y=aU+cW.
\]
In this construction, the important feature is coordinate 0. On task $A$, the variable $U$ appears in the input, but does not appear in the output, meaning that the ideal predictor should ignore coordinate 0. In contrast, on task $B$, the same variable $U$ is useful, and should receive coefficient $a$. Therefore, in this case, \textit{the two tasks disagree about the desired use of the same input coordinate} (a simple scenario of disagreements between tasks). 

The situation with other coordinates is simpler. Coordinate $1$, with variable $V$, carries the useful signal for task A, while coordinate 2 carries the useful $W$-signal for task $B$. In this case, for simplicity, we assume that the observations are noiseless, and we study \textit{population} squared risk rather than finite-sample estimation error (which is a proxy to evaluate the population risk). In particular, \textit{this ensures that any beneficial ablation in the analysis comes from task conflict and training dynamics, rather than from overfitting to a finite sample}.

\paragraph{Attention architecture and input encoding.} We next consider a simple one-layer Transformer-style attention model with three attention heads. In this case, our goal is to understand how training allocates already available features across heads, rather than how features themselves are learned. We use a fixed coordinate encoding as follows. Let $e_0, e_1, e_2 \in \bbR^3$ denote the standard basis vectors. For an input $x = (x_0, x_1, x_2) \in \{-1, 0, 1\}^3$, we associate coordinate $j$ with the query $z_j(x) = \sqrt{3} x_je_j$. This means that the original input $x$ will be encoded as the series of tokens before being fed to the transformer
\[
    z(x) = \begin{pmatrix}
        \sqrt{3}x_0 & 0 & 0\\
        0 & \sqrt{3} x_1 & 0\\
        0 & 0 & \sqrt{3} x_2.
    \end{pmatrix}
\]
Therefore, the three coordinates of $x$ are represented by three fixed feature directions. This is a \textit{deliberate modeling choice}, as we want that the feature basis is supplied to the model, while how these features are associated to the attention heads is learned during the training process.

\paragraph{Fixed memory bank}. For every coordinate $j \in \{0, 1, 2\}$, we introduce a fixed pair of memory tokens $m_{j, +} = e_j, \quad m_{j, -} = - e_j.$
The query $z_j(x)$ attends only to the corresponding pair $(m_{j, +}, m_{j, -})$, using a fixed attention mask shared by all three heads. This paired construction serves two purposes. First, it gives every head exactly the same access to \textit{every} feature, so no feature-to-head assignment is deterministically imposed in advance. Second, the symmetric pair $(+ e_j, -e_j)$ yields a particular signed softmax response (which we will see below), allowing us to analyze the training dynamics without replacing softmax attention by an artificial surrogate mechanism. 

\paragraph{Attention output.} Given an input $x = (x_0, x_1, x_2)$, we form three queries $(z_0(x), z_1(x), z_2(x))$. For each head $h \in \{1, 2, 3\}$ and query $z_j(x)$, we use the standard scaled dot-product attention over the corresponding memory pair $(m_{j, +}, m_{j, -})$. Specifically, for $\sigma \in \{-1, 1\}$, the attention weight $A_{h, j, \sigma}$ can be calculated as
\[
    A_{h, j, \sigma}(x) = \frac{\exp \left(\inner{Q_{h}z_j(x), K_h m_{j, \sigma}} / \sqrt{3}\right)}{\Sigma_{\tau \in \{\pm 1\}} \exp \left(\inner{Q_{h}z_j(x), K_h m_{j, \tau}} / \sqrt{3}\right)}.
\]
Roughly speaking $A_{h, j, +}(x)$ and $A_{h, j, -}(x)$ are simply the two softmax weights assigned by head $h$ to the positive and negative memory token associated with feature $j$. The corresponding value vectors are $V_hm_{j, +}$ and $V_{h}m_{j, -}$, and the output $g_{h, j}(x)$ of head $h$ for the query $j$ can be written as
\[
    g_{h, j}(x) = O_h \left(\sum_{\sigma \in \{-1, 1\}} A_{h, j, \sigma}(x) V_h m_{j, \sigma}\right) \in \bbR^3. 
\]
We then sum over all feature queries and all three attention heads: $ g(x) = \sum_{h = 1}^3 \sum_{j = 0}^2 g_{h, j}(x)  \in \bbR^3.$
Because our prediction target $Y$ is scalar, we finally apply the fixed readout $r(z) = \mathbf{1}^\top z$, and the scalar prediction of the model is $ f_{\theta}(x) = \mathbf{1}^\top g(x),$
where $\theta$ denotes the tuple containing all the trainable parameters $Q_h, K_h, V_h, O_h$, for $h \in \{1, 2, 3\}$.

\paragraph{Diagonal parameterization.} Following prior works of Transformer training dynamics (e.g., \citep{boix2023transformers}), for simplicity and tractability of the analysis, we restrict the query, key, value, and output projection matrices of each head to be diagonal, that is
\begin{equation*}
    \begin{aligned}
        &Q_h = \diag(q_{h0}, q_{h1}, q_{h2}), K_h = \diag(k_{h0}, k_{h, 1}, k_{h2}), \\
        &V_h = \diag(v_{h0}, v_{h1}, v_{h2}), O_h = \diag(o_{h0}, o_{h1}, o_{h2}).
    \end{aligned}
\end{equation*}
This restriction removes cross-coordinate interactions and allows us to focus on analyzing the competition among heads for each feature. Importantly, we note that this assumption does NOT make the optimization dynamics linear, as the query, key, value, and output parameters are still trained separately and interact multiplicatively through the softmax attention mechanism. 

\paragraph{Training objective.} Since our goal is to understand whether joint training can create the capability enhancement localization phenomenon, we train the same attention model jointly on tasks $A$ and $B$. For task $T \in \{A, B\}$, let $R_T(\theta) = \bbE_{T}[(f_\theta(X) - Y)^2]$ denote the population squared risk of the attention model on task $T$, where $f_\theta$ is the scalar predictor defined by the attention block above. Since the two considered tasks are equally weighted, we train the predictor $f_\theta$ using the half-squared population loss, that is $\cL(\theta) = \frac{1}{4}(R_A(\theta) + R_B(\theta).$
\paragraph{Weight initialization and training algorithm.} We train all query, key, value, and output parameters using fixed step-size \textit{full-batch} gradient descent. Concretely, for each head $h \in \{1, 2, 3\}$ and feature $j$, we collect the four trainable scalar parameters into $\theta_{hj} = (q_{hj}, k_{hj}, v_{hj}, o_{hj})$. We initialize the weights $\theta_{hj}$ as $\theta^0_{hj}$, where
\[
    \theta^0_{hj} = \alpha G_{hj}, \quad G_{hj} \overset{i.i.d}{\sim} N(0, I_4),
\]
where $\alpha > 0$ is a small positive constant controlling the initialization scale. In other words, the weights $\theta_{hj}$ at initialization have covariance $\alpha^2 I_4$. Starting from this initialization, all parameters are updated using full-batch GD
\[
    \theta^{n + 1} = \theta^n - \eta \nabla \cL(\theta^n),
\]
where $\eta > 0$ is a fixed learning rate. Importantly, no feature is assigned to any head beforehand, and we do NOT perform sign conditioning, parameter tying, rebalancing, or any other intervention mechanism during training. This means that \textit{if a particular feature eventually becomes concentrated in one attention head, that localization must arise naturally from the training dynamics and the random initialization rather than from predetermined head assignment.}

\subsection{Joint Optimality Does Not Determine Head Allocation} \label{sec:joint-optimality-closed-form-formulation}
We first express the attention model in terms of its effective feature coefficients. This representation identifies the predictor that minimizes the joint population loss and separates this aggregate optimum from the allocation of features across individual heads.

\begin{proposition}[Exact softmax reduction and coefficient compromise] \label{prop:simple-representation}
    Consider the three-head attention model described as in Section~\ref{sec:attention-model-setting}, with the fixed coordinate encoding, paired memory bank, and diagonal query, key, value and output projection matrices. For each head $h \in \{1, 2, 3\}$ and feature $j \in \{0, 1, 2\}$, define the effective head-feature coefficient $c_{hj} = o_{hj} v_{hj}\tanh(q_{hj} k_{hj})$, and let $w_j = \sum_{h = 1}^3c_{hj}$ denote the total coefficient assigned to feature $j$ after summing over three heads. Then we have
    \begin{enumerate}[leftmargin=*]
        \item \textbf{Head-feature response:} For each input $x = (x_0, x_1, x_2)$, the contribution of head $h$ associated with feature $j$ is exactly $x_j c_{hj}e_j$.
        \item \textbf{Predictor representation}: The model prediction can be written as $f_\theta(x) = w^\top x$, where $w = (w_0, w_1, w_2)$ and $w_i$ is defined as above. In other words, the resulting predictor is linear in the input features. However, the effective coefficients $w_{j}$ are highly non-linear functions of the trainable weights.
        \item \textbf{Population risks:} Under the two data distributions defined as in Section~\ref{sec:attention-model-setting}, the population risks reduce to 
        \[
            R_A(w) = w_0^2 + (w_1 - b)^2, \quad R_B(w) = (w_0 - a)^2 + (w_2 - c)^2.
        \]
        \item \textbf{Joint training objective and optimum:} As a result, the population objective $\cL(\theta) = \frac{1}{4}\left(R_A(w(\theta)) + R_B(w(\theta))\right)$ can be rewritten as
        \[
            \cL(\theta) = \frac{1}{2}\left(w_0 - \frac{a}{2}\right)^2 + \frac{1}{4}(w_1 - b)^2 + \frac{1}{4}(w_2 - c)^2 + \frac{a^2}{8}.
        \]
        This implies the optimal aggregate coefficient vector $w^\star = \beta = \left(\frac{a}{2}, b, c\right)$.
    \end{enumerate}
\end{proposition}
\textit{Proof sketch.}
    The proof of Proposition~\ref{prop:simple-representation} is straightforward from the original problem setting (Section~\ref{sec:attention-model-setting}). The detailed proof can be found in Appendix~\ref{apx:theoretical-explanation}.
\qed

\textbf{Task conflict produces a compromise.} Within this model, the coefficient on (U) is shared across tasks and cannot depend on which task generated the input. Task (A) requires coefficient zero, whereas task (B) requires coefficient (a). Joint optimization therefore selects $w_0^\star=a/2$, while fitting the task-specific coefficients (b) and (c) exactly. Both tasks have risk $a^2/4$ at this joint optimum.

\textbf{Head allocation remains undetermined.} Although the optimal aggregate vector is unique, its decomposition across heads is not: the objective constrains only the sums $w_j=\sum_hc_{hj}$. This distinction matters because ablating head $h$ changes the predictor’s coefficient vector from $w$ to $w - c_h$, where $c_h=(c_{h0},c_{h1},c_{h2})$. Removing a head that carries the conflicting coordinate may therefore also remove useful task-specific signal. Establishing beneficial head ablation requires understanding which decomposition training selects. We next show that gradient descent from small random initialization concentrates each feature in a single head.

\subsection{Gradient descent localizes features} \label{sec:gd-localize-feature}
We now show that gradient descent selects a particular structure among the many head decompositions compatible with the optimal aggregate predictor. Under small Gaussian initialization, each feature becomes concentrated in one head, while its contributions from the remaining heads stay small throughout training.

\begin{theorem}[Training-induced feature localization] \label{thm:training-induce-localization}
    For $h \in \{1, 2, 3\}$, $j \in \{0, 1, 2\}$, let
    \[
        \theta_{hj} = (q_{hj}, k_{hj}, v_{hj}, o_{hj}), \quad c_{hj} = o_{hj}v_{hj}\tanh(q_{hj}k_{hj}), \quad w_j  = \sum_{h = 1}^3 c_{hj}.
    \]
    Fix $a, b, c > 0$, set $\beta = \left(\frac{a}{2}, b, c\right)$ and $p = \left(1, \frac{1}{2}, \frac{1}{2}\right)$, and train $\cL(\theta) = \frac{1}{2}\sum_{j = 0}^2 p_j (w_j - \beta_j)^2 + \frac{a^2}{8}$ by full-batch GD from the initialization $\theta^0_{hj} = \alpha G_{hj}$, where $G_{hj} \overset{i.i.d.}{\sim} N(0, I_4)$. Then there exists $\eta_\star > 0$, depending only on $a, b, c$, such that the following holds for every fixed step size $0 < \eta < \eta_\star$. For almost every draw of $G$, and for all sufficiently small $\alpha > 0$, we have:
    \begin{enumerate}[leftmargin=*]
        \item \textbf{Convergence:} the parameter trajectory is bounded and converges to a finite limit with $w^\infty = \beta$. Each residual $r_j^n = \beta_j - w_j^n$ is positive and non-increasing.
        \item \textbf{Localization:} there are labels $h_j = h_j(G)$, independent and uniform on $\{1, 2, 3\}$, and a finite constant $K_G$, independent of $\alpha$, such that 
        \[
            \sup_{n \geq 0}\|\theta^n_{hj}\|_\infty \leq K_G\alpha, \quad \sup_{n \geq 0}|c^n_{hj}| \leq K_G\alpha^4 \textup{ for } h \neq h_j. 
        \]
        Consequently, $c^\infty_{h_jj} = \beta_j + \cO_G(\alpha^4)$.
        \item \textbf{Finite checkpoint:} define $N_\alpha = \lceil \frac{\log(1/\alpha)}{\eta \alpha^2}  \rceil$, for every fixed $M > 0$ and sufficiently small $\alpha$, we have $\sup_{n \geq N_\alpha} \max_{j} r_j^n \leq \alpha^M.$
    \end{enumerate}
\end{theorem}
The proof details of Theorem~\ref{thm:training-induce-localization} can be found in Appendix ~\ref{apx:theoretical-explanation}. Intuitively, Theorem~\ref{thm:training-induce-localization} establishes localization through training without prescribing a feature-to-head assignment. Every head initially has access to every feature, but random initialization and gradient descent select a winning head for each feature. This is a feature-wise statement: different features may share the same winning head. Thus, localization alone does not yet guarantee beneficial head ablation. The next subsection identifies an allocation event under which removing the head carrying the conflicting feature preserves the task-specific features and improves task $A$.  

\subsection{Localization Enables Beneficial Head Ablation} \label{sec:localization-enable-beneficial-abalation}
We now connect feature localization to task-specific improvement. If the conflicting feature is concentrated in a head that carries neither task-specific feature, ablating that head removes the unwanted contribution to task $A$ while approximately preserving its useful signal. The independent winning-head assignments in Theorem~\ref{thm:training-induce-localization} give this event a positive probability. The detailed proofs of Theorem~\ref{thm:training-induce-ablation} can be found in Appendix~\ref{apx:theoretical-explanation}.

\begin{theorem}[Training-induced beneficial head ablation] \label{thm:training-induce-ablation}
    Under the assumptions of Theorem~\ref{thm:training-induce-localization}, let $t = a / 2$, and define the event
    \[
        E = \{h_0 \not \in \{h_1, h_2\}\},
    \]
    where $h_j$ is the head that wins the feature $j \in \{0, 1, 2\}$. Then $\Pr(E) = \frac{4}{9}$. Under such event, set $h_\star = h_0$, and delete this head without changing any other parameters
    \[
        f^n_{-h_\star} = f^n - f^n_{h_\star}.
    \]
    For almost every initialization draw $G$ in $E$, all sufficiently small $\alpha>0$, and every $n\geq N_\alpha$, 
    \begin{equation*}
        \begin{aligned}
            &G_A^n \triangleq R_A(f^n) - R_A(f^n_{-h_\star}) = t^2 + \cO_G(\alpha^4) > 0, \\
            &D_B^n \triangleq R_B(f^n_{-h_\star}) - R_B(f^n) = 3t^2 + \cO_G(\alpha^4) > 0,
        \end{aligned}
    \end{equation*}
    where the error bounds are uniform over $n \geq N_\alpha$. Consequently
    \[
        \liminf_{\alpha \rightarrow 0^+} \Pr_{G}\left(\exists h: R_A(f^{N_\alpha}) - R_A(f^{N_\alpha}_{-h}) \frac{t^2}{2}, R_B(f_{-h}^{N_\alpha}) - R_B(f^{N_\alpha}) \geq \frac{3t^2}{2}\right) \geq 
        \frac{4}{9}.
    \]
\end{theorem}

\paragraph{Interpretation of Theorem~\ref{thm:training-induce-ablation}.} In the small-initialization limit, both intact task risks approach $t^2$. Removing the selected head reduces task A's risk to zero but increases task B's risk to $4t^2$. Consequently, the joint objective increases by $t^2/2$. The intervention therefore improves one task by moving away from the jointly optimal compromise. By parameter convergence, the same ablation conclusions also hold at the limiting trained model. This therefore theoretically describes the CEL phenomenon.

\section{Discussion and Limitations}
\textbf{Empirical scope.} Our experiments provide evidence of beneficial sparse head ablations across three model families and four benchmarks. These findings concern the evaluated models, datasets, and inference protocols; they do not establish that beneficial ablations exist for every model–task pair. Although selection and evaluation use disjoint examples, the stability of the selected subsets and their gains across different discovery samples remains insufficiently characterized. 

\textbf{Search and head interactions.} Our two-stage procedure identifies beneficial interventions but does not guarantee an optimal subset. Individual-head screening can exclude combinations whose benefits emerge only through joint ablation. Consequently, the gains depend on both the discovery data and the candidate pool, and an unsuccessful search would not establish the absence of beneficial sparse interventions.

\textbf{Theoretical scope.} We note that our analysis establishes a possible training-induced mechanism for CEL under a fixed feature encoding, diagonal projections, small Gaussian initialization, and full-batch gradient descent on a noiseless population objective. The architecture supplies the feature directions and restricts cross-coordinate interactions, while training determines their allocation across heads. The resulting task conflict arises within the shared predictor class. Therefore, the theory establishes one possible mechanism through which beneficial head ablation can arise despite joint optimality. It does not establish that this mechanism explains the empirical gains in the evaluated language models, nor that it accounts for all instances of CEL.

\section{Conclusion and Future Work}
In this work, we study Capability Enhancement Localization (CEL): whether removing a small set of attention heads can improve task-specific performance without further fine-tuning. Across three model families and four benchmarks, a simple two-stage search identifies sparse head subsets whose ablation improves held-out accuracy. Motivated by these observations, we establish a potential training-induced mechanism for CEL in a simplified softmax attention model. Gradient descent from small Gaussian initialization concentrates features in individual heads, and, with positive probability, removing the head carrying a conflicting feature improves one task while harming another. This can occur even at a global minimum of the joint training objective, showing that successful joint optimization is compatible with beneficial task-specific head ablation.

We note that this work only provides an initial empirical and theoretical investigation of CEL. On the empirical side, future work could develop more efficient search procedures that account for interactions among heads, including beneficial combinations that individual-head screening may overlook. On the theoretical side, an important direction is to characterize when training produces beneficial head ablations beyond the fixed feature encoding and diagonal parameterization studied here. In particular, allowing learned representations and interactions among features could clarify when conflicting contributions become sufficiently separated from useful ones for sparse ablation to improve task performance.
\newpage 
\subsection*{AI use statement}


\textbf{Experiment implementation.} For the codes of all experiments described in this work, the authors use OpenAI Codex (with models GPT-5.6 Sol, GPT-6 Astra) for implementations. 

\textbf{Theoretical results.} For the theoretical models, claims, and construction provided in this work, the authors use GPT-6.0 Astra for brainstorming ideas, theorem construction, and theorem proving, though the authors already had clue about the constructions at the first place. After that, the authors organize the statements, the proofs, and the elaboration in this paper with the help of LLMs. Finally, when incorporating the statements/proofs/elaborations into the main body and appendix, the authors rewrite the contents based on the authors' understanding. The authors verified the soundness and correctness of the theoretical results presented in this draft. 

\textbf{Paper writing.} For paper writing, the authors use GPT-5.6 Sol and Grammarly. For every paragraph in this paper, the authors use the following workflow: (1) the authors draft the paragraph without the use of LLMs, (2) the authors then use GPT-5.6 Sol to improve/correct the wording and grammars, (3) the authors again modify the paragraph (fixed by GPT-5.6 Sol) again to improve the flow and personalize the deliver, and (4) the authors finally use Grammarly to fix the typos and grammars once again. The authors provide the citations and references by hand in the first step using Google Scholar, and no LLM is used for generating citation.

\bibliography{references}
\bibliographystyle{iclr2027_conference}

\include{iclr2027_appendix}

\end{document}

%% file: math_commands.tex
\usepackage{amsmath,amsfonts,bm}

\def\eqref#1{equation~\ref{#1}}

\def\1{\bm{1}}

\DeclareMathAlphabet{\mathsfit}{\encodingdefault}{\sfdefault}{m}{sl}
\SetMathAlphabet{\mathsfit}{bold}{\encodingdefault}{\sfdefault}{bx}{n}



%% file: iclr2027_appendix.tex
\appendix





\section{Additional Information for Experiment Setting}

\begin{table}[H]
      \small
      \centering
      \caption{Model details. Here, $H$ denotes query heads per layer,
      and $N = LH$ is the total number of ablatable heads.
      In our experiments, for convenience, SmolLM3 is evaluated in non-thinking mode.}
      \label{tab:model_details}
      \begin{tabular}{lccc}
      \toprule
      \textsc{Model} & \textsc{Layers} $L$ & \textsc{Heads} $H$ & \textsc{Total Heads} $N$ \\
      \midrule
      LLaMA-3.2-3B & 28 & 24 & 672 \\
      Qwen2.5-3B  & 36 & 16 & 576 \\
      SmolLM3-3B  & 36 & 16 & 576 \\
      \bottomrule
      \end{tabular}
  \end{table}

\begin{table}[H]
  \small
  \centering
  \caption{Dataset sizes for discovery $D_\textup{dist}$ and held-out evaluation $D_\textup{test}$. }
  \label{tab:dataset_sizes}
  \begin{tabular}{lrr}
  \toprule
  \textsc{Dataset} & \textsc{Discovery} & \textsc{Held-out} \\
  \midrule
  Arithmetic & 2,000 & 1,000 \\
  HellaSwag  & 2,000 & 10,042 \\
  BoolQ      & 2,000 & 3,270 \\
  MBPP       & 374  & 500 \\
  \bottomrule
  \end{tabular}
\end{table}
 
\section{Additional Results and Omitted Proofs for Section~\ref{sec:theoretical-explanation}} \label{apx:theoretical-explanation} 
In this section, we will provide the detailed proof for the results presented in Section~\ref{sec:theoretical-explanation}.

\subsubsection{Omitted Proof for Section~\ref{sec:joint-optimality-closed-form-formulation}}
 First,  we will provide the detailed proof for Proposition~\ref{prop:simple-representation}, which provides simplification for the feature response, predictor representation, population risks, and training optimum in our setting.

\begin{customprop}{\ref{prop:simple-representation} (restated)}
    Consider the three-head attention model described as in Section~\ref{sec:attention-model-setting}, with the fixed coordinate encoding, paired memory bank, and diagonal query, key, value and output projection matrices. For each head $h \in \{1, 2, 3\}$ and feature $j \in \{0, 1, 2\}$, define the effective head-feature coefficient $c_{hj} = o_{hj} v_{hj}\tanh(q_{hj} k_{hj})$, and let $w_j = \sum_{h = 1}^3c_{hj}$ denote the total coefficient assigned to feature $j$ after summing over three heads. Then we have
    \begin{enumerate}[leftmargin=*]
        \item \textbf{Head-feature response:} For each input $x = (x_0, x_1, x_2)$, the contribution of head $h$ associated with feature $j$ is exactly $x_j c_{hj}e_j$.
        \item \textbf{Predictor representation}: The model prediction can be written as $f_\theta(x) = w^\top x$, where $w = (w_0, w_1, w_2)$ and $w_i$ is defined as above. In other words, the resulting predictor is linear in the input features. However, the effective coefficients $w_{j}$ are highly non-linear functions of the trainable weights.
        \item \textbf{Population risks:} Under the two data distributions defined as in Section~\ref{sec:attention-model-setting}, the population risks reduce to
        \[
            R_A(w) = w_0^2 + (w_1 - b)^2, \quad R_B(w) = (w_0 - a)^2 + (w_2 - c)^2.
        \]
        \item \textbf{Joint training objective and optimum:} As a result, the population objective $\cL(\theta) = \frac{1}{4}\left(R_A(w(\theta)) + R_B(w(\theta))\right)$ can be rewritten as
        \[
            \cL(\theta) = \frac{1}{2}\left(w_0 - \frac{a}{2}\right)^2 + \frac{1}{4}(w_1 - b)^2 + \frac{1}{4}(w_2 - c)^2 + \frac{a^2}{8}.
        \]
        This implies the optimal aggregate coefficient vector $w^\star = \beta = \left(\frac{a}{2}, b, c\right)$.
    \end{enumerate}
\end{customprop}
\proof 
    We will present the proof in the following steps. 

    \textbf{Step 1: Head-feature response:} For a fixed head $h$ and a feature $j$, we recall that the query associated with feature $j$ is $z_j(x) = \sqrt{3}x_je_j$, while the two corresponding memory tokens are $m_{j, +} = e_j, m_{j, -} = -e_j$. From assumptions, the key and query matrices are diagonal, we have $Q_h z_j(x) = \sqrt{3} x_j q_{hj}e_j$, $K_h m_{j, +} = k_{hj}e_j$, and $K_{h}m_{j, -} = -k_{hj}e_j$. Therefore, after the $(1/\sqrt{3})$-scaling in the dot-product attention, the two logits are
    \[
        +x_j q_{hj}k_{hj}, \quad -x_j q_{hj}k_{hj}.
    \]
    This means that the two softmax weights are
    \[
        A_{h, j, +}(x) = \frac{e^{x_j q_{hj} k_{hj}}}{e^{x_jq_{hj}k_{hj}} + e^{-x_jq_{hj}k_{hj}}}, \quad A_{h, j, -}(x) = \frac{e^{-x_j q_{hj} k_{hj}}}{e^{x_jq_{hj}k_{hj}} + e^{-x_jq_{hj}k_{hj}}}.
    \]
    The corresponding value vectors are $V_{h}m_{j, +} = v_{hj}e_j$, $V_h m_{j, -} = -v_{hj}e_j$, and the softmax weighted combination is
    \[
        v_{hj} \frac{e^{x_jq_{hj}k_{hj}} - e^{-x_jq_{hj}k_{hj}}}{e^{x_jq_{hj}k_{hj}} + e^{-x_jq_{hj}k_{hj}}} = v_{hj} \tanh(x_j q_{hj} k_{hj})e_j. 
    \]
    After applying the output projection $O_h$, the contribution of the head $h$ for feature $j$ is $o_{hj}v_{hj}\tanh(x_j q_{hj}k_{hj})e_j$. Note that $x_j \in \{-1, 0, 1\}$, and tanh is an odd function, we have 
    \[
        g_{h, j}(x) = x_j c_{hj} e_j, \quad c_{hj} = o_{hj}v_{hj}\tanh(q_{hj}k_{hj}),
    \]
    as expected. 

    \textbf{Step 2: Predictor representation:} We now sum over all heads to have $ w_j = \sum_{h = 1}^3 c_{hj}.$
    After summing over all feature queries and all heads, and applying the scalar readout, the model prediction becomes $f_\theta(x) = w^\top x$ as expected. Here, we can see that the model is linear in the input $x$. However, the coefficient $w_j$ are not directly trained parameters, but complicated functions of parameters
    \[
        w_j = \sum_{h = 1}^3 c_{hj}, \quad c_{hj} = o_{hj}v_{hj}\tanh(q_{hj}k_{hj}).
    \]
    This is an important observation for the later localization result. 

    \textbf{Step 3: task-specific population risks and joint training objective:} Using the definition of task $A$ and $B$, the closed form of the predictor $f_\theta(x)$, and the fact that the feature is Rademacher variables, we can directly calculate $R_A(w) = w_0^2 + (w_1 - b)^2$ and $R_B(w) = (w_0 -a)^2 + (w_2 - c)^2$. Therefore, the joint training objective is
    \[
        \cL(\theta) = \frac{1}{2}\left(w_0 - \frac{a}{2}\right)^2 + \frac{1}{4}(w_1 - b)^2 + \frac{1}{4}(w_2 - c)^2 + \frac{a^2}{8},
    \]
    and trivially the optimal aggregated coefficient $w^\star = \beta = \left(\frac{a}{2},b,c\right)$.
\qed

\subsection{Additional Results and Omitted Proofs for Section~\ref{sec:gd-localize-feature}}
In this section, we will present the formal statements of supporting results and Theorem~\ref{thm:training-induce-localization}. The first result essentially says that, under small initialization and \textit{gradient flow} trajectory (i.e., infinitesimal step-size) and for a given feature, there will be a \textit{corresponding leader} head of which the weights will move away from initialization in a certain amount of time. Moreover, the earliest escaping time will be \textit{unique}, i.e., there will not be two heads being the leaders for a single feature at the same time. 

\begin{lemma}[Early competition among heads under small initialization] \label{lm:early-competition-among-heads}
    Fix a feature $j \in \{0, 1, 2\}$. For each head $h \in \{1, 2, 3\}$, the four trainable parameters associated with this feature are $\theta_{hj} = (q_{hj}, k_{hj}, v_{hj}, o_{hj}) \in \bbR^4$. Recall that they are initialized as
    \[
        \theta^0_{hj} = \alpha G_{hj}, \quad G_{hj} \overset{\mathrm{iid}}{\sim} N(0, I_4),
    \]
    where $\alpha > 0$ is the small initialization scale. To describe their leading-order behavior near the (small) initialization, for a generic initial vector $g = (g_q, g_k, g_v, g_o) \in \bbR^4$, consider the auxiliary ODE
    \begin{equation} \label{eq:aux_ode}
        \begin{cases}
            \dot{u}_q = u_ku_vu_o, \\
            \dot{q}_k = u_qu_vu_o, \\
            \dot{u}_v = u_qu_ku_o, \\
            \dot{u}_o = u_qu_ku_v,
        \end{cases} \quad u(0) = g,
    \end{equation}
    where $u(t) = (u_q(t), u_k(t), u_v(t), u_o(t))$ describe the leading-order evolution of the normalized quartet $\theta_{hj} / \alpha$. Assume that the coordinate $g$ are nonzero and that one coordinate has strictly smaller magnitude than the other three. Then: 
    \begin{enumerate}[leftmargin=*]
        \item \textbf{Finite auxiliary escape time}: The ODE in Eq.~\ref{eq:aux_ode} has a unique maximal solution $u_g:[0, T(g)) \rightarrow \bbR^4$, for some finite $T(g)$, and $\|u_g(t)\| \rightarrow \infty$ as $t \rightarrow T(g)^-$. We call $T(g)$ the \textbf{auxiliary escape time} of the quartet.
        \item \textbf{Sign alignment and balancing}: the product $u_q(t) u_k(t) u_v(t) u_o(t)$  eventually becomes positive, and as $t \rightarrow T(g)^-$, the weights are balanced
        \[
            \frac{u_r(t)^2}{u_s(t)^2} \rightarrow 1, \quad \forall r, s \in \{q, k, v, o\}.
        \]
        \item \textbf{Unique leader for each feature}: For three normalized Gaussian initialization $G_{1j}, G_{2j}, G_{3j} \overset{\mathrm{iid}}{\sim} N(0, I_4)$, the escape times $T(G_{1j}), T(G_{2j}), T(G_{3j})$ are almost surely finite and distinct. Therefore $\hat{h}_j \triangleq \arg \min_{h \in \{1, 2, 3\}} T(G_{hj})$ is almost surely uniquely defined. Moreover, $\Pr(\hat{h}_j = 1) = \Pr(\hat{h}_j = 2) = \Pr(\hat{h}_j = 3) = \frac{1}{3}$. The variable $\hat{h}_0, \hat{h}_1, \hat{h}_2$ are independent across features
    \end{enumerate}
\end{lemma}

\proof 
    We will organize the proof using the following steps.

    \textbf{Step 1: why the four-dimensional system is solvable.} Fix one normalized initialization $g = (g_q, g_k, g_v, g_o)$, and let $u(t) = (u_q(t), u_k(t), u_v(t), u_o(t))$ be the solution of Eq.~\ref{eq:aux_ode}. For any coordinate $r \in \{q, k, v, o\}$, we have $\frac{d}{dt} u_r(t)^2 = 2u_r(t)\dot{u}_r(t)$. Note that $\frac{d}{d_t}u_q^2 = 2u_q(u_ku_vu_o) = 2u_qu_ku_vu_o$, and the same identity holds for every coordinate:
    \begin{equation*} \label{eq:derivative-similiarity}
        \frac{d}{d_t}u_q^2 = \frac{d}{d_t}u_k^2 = \frac{d}{d_t}u_v^2 = \frac{d}{d_t}u_o^2 =  2u_qu_ku_vu_o.
    \end{equation*}
    Therefore, the differences between squared coordinates are constant:
    \[
        u_r(t)^2 - u_s(t)^2 = g_r^2 - g_s^2.
    \]
    Define $a_r \triangleq g_r^2$, then there exists a single scalar function $\xi(t)$ with $\xi(0) = 0$ such that 
    \begin{equation} \label{eq:ode-solution}
        u_r(t)^2 = a_r + \xi(t), \quad \forall r \in \{q, k, v, o\}.
    \end{equation}
    \textbf{Step 2: dynamic of the common scalar $\xi(t)$.} Let $P(t) \triangleq u_q(t)u_k(t)u_v(t)u_o(t)$ be the coordinate product. By Step 1 and Eq.~\ref{eq:ode-solution}, we have $\dot{\xi}(t) = 2P(t)$, and from Eq.~\ref{eq:ode-solution}, we have $P(t)^2 = \prod_{r \in \{q, k, v, o\}}(a_r + \xi(t))$. Therefore
    \begin{equation}
        \dot{\xi} = \begin{cases}
            2 \sqrt{\prod_{r} (a_r + \xi)}, \quad P(t) > 0,\\
            -2 \sqrt{\prod_{r}(a_r + \xi)}, \quad P(t) < 0.
        \end{cases}
    \end{equation}
    This means that the four-dimensional ODE has reduced to a one-dimensional separable equation as above. 
    
    \textbf{Step 3: Positive initial products.} Let $P_0 \triangleq g_qg_kg_vg_0$. First, supposed that $P_0 > 0$, then initially we have $\dot{\xi} > 0$. Since all $u_r^2 = a_r + \xi$ are increasing, no coordinate can cross zero, which means $P(t)$ says positive.  Therefore, $\xi(t)$ increases monotonically from $0$ toward $+ \infty$, and separation of variables gives
    \[
        t = \int_{0}^{\xi(t)} \frac{ds}{2\sqrt{\prod_r(a_r + s)}}.
    \]
    Define $I_+(g) \triangleq \int_{0}^\infty \frac{ds}{2\sqrt{\prod_{s}(a_r + s)}}$, and for large $s$, note that $\sqrt{\prod_{r}(a_r + s)} \approx s^2$, so the integrand behaves as $\frac{1}{2s^2}$. Therefore $I_+(g) < \infty$, and $T(g) = I_+(g)$ when $P_0 > 0$. Moreover, we have $\xi(t) \rightarrow \infty$ as $t \rightarrow T(g)^-$, and therefore
    \[
        u_r(t)^2 = a_r + \xi(t) \rightarrow \infty.
    \]
    This means that $\|u_g(t)\| \rightarrow \infty$.

    \textbf{Step 4: Negative initial products.} Now assume that $P_0 < 0$. Then initially, $\dot{\xi} < 0$, so $\xi$ decreases. Let $m \triangleq \min_r a_r$. By assumption, this minimum is achieved by a unique coordinate, say $r_\star$. From $u_{r_\star}(t)^2 = m + \xi(t)$, that coordinate reaches 0 when $\xi = -m$. All other coordinate remain nonzero because their corresponding $a_r$'s are strictly larger than $m$. At the point where $u_{r_\star} = 0$, we have 
    \[
        \dot{u}_{r_\star} = \prod_{s \neq r_\star} u_s(t) \neq 0,
    \]
    which means the coordinate does not remain at zero but crossing through zero and changes sign. This change flips the sign of $P(t) = u_qu_ku_vu_o$ from negative to positive. Therefore, after the cross, $\dot \xi > 0$, and the scalar trajectory follows $0 \rightarrow -m \rightarrow 0 \rightarrow +\infty$. We define
    \[
        I_-(g) \triangleq \int_{-m}^0 \frac{ds}{2\sqrt{\prod_r (a_r + s)}} > 0. 
    \]
    Near $s = -m$, exactly one factor in the denominator vanishes linearly, so the integrand behave like $(s + m)^{-1/2}$, which is integrable, and therefore $0 < I_-(g) < \infty$. Now note that the trip from $0$ to $-m$ takes time $I_-$, the trip from $-m$ to $0$ take time $I-$, and the trip from $0$ to $\infty$ take time $I_+$ as shown above. Therefore, $T(g) = I_+(g) + 2I_-(g)$ when $P_0 < 0$. Hence, we conclude that in both case, $T(g) < \infty$.

    \textbf{Step 5: Why the four factor magnitudes become balanced.} From the above, we know that $u_r(t)^2 = a_r + \xi(t)$, and therefore
    \[
        \frac{u_r(t)^2}{u_s(t)^2} = \frac{a_r + \xi(t)}{a_s + \xi(t)}. 
    \]
    As $t \rightarrow T(g)^-$, we have $\xi(t) \rightarrow \infty$, and therefore $\frac{u_r(t)^2}{u_s(t)^2} \rightarrow 1$.

    \textbf{Step 6: Scaling the initialization.} We have this useful observation: if the initial vector is multiplied by $r > 0$,t hen $a_i \rightarrow r^2 a_i$. Changing variables $s = r^2 z$ in the integrals above gives $T(rg) = r^{-2}T(g)$. This is also consistent with the actual initialization $\theta^0_{hj} = \alpha G_{hj}$. If we applied the auxiliary ODE directly to the initialization $\alpha G_{hj}$, its escape time would be $T(\alpha G_{hj}) = \alpha^{-2}T(G_{hj})$. But the common factor $\alpha^{-2}$ is identical for all three heads. Therefore
    \[
        \arg \min_{h} T(\alpha G_{hj}) = \arg \min_{h} T(G_{hj}).
    \]
    This is why we can skip the normalized factor $\alpha$ and compare the normalized escape times $T(G_{hj})$.

     \textbf{Step 7: Gaussian escape times have no atoms.} Let $G \sim N(0, I_4)$, and write $G = R\Omega$, where $R = \|G\| = 0$, and $\Omega = \frac{G}{\|G\|}$. For an isotropic Gaussian vector, $R$ and $\Omega$are independent, and $R$ has a continuous density. From Step 6, we have $T(G) = T(R\Omega) =  \frac{T(\Omega)}{R^2}$. With probability one, all coordinates of \(\Omega\) are nonzero and its smallest absolute coordinate is unique. By the preceding steps, \(T(\Omega)\in(0,\infty)\) for every such direction. Conditional on any such \(\Omega\), \(T(\Omega)\) is fixed, and the mapping \(R\mapsto T(\Omega)/R^2\) is strictly decreasing and continuous on \((0,\infty)\). Therefore, for every fixed $t > 0$, we have
     \[
        \Pr(T(G) = t \mid \Omega) = 0.
     \]
     Averaging over $\Omega$, we have $\Pr(T(G) = t) = 0$. 

     \textbf{Step 8: Unique and uniformly early leader.} For a fixed feature $j$, we have
     \[
        G_{1j}, G_{2j}, G_{3j} \overset{\mathrm{iid}}{\sim} N(0, I_4).
     \]
     Therefore, $T(G_{1j}), T(G_{2j}), T(G_{3j})$ are i.i.d. atom-less random variables. Therefore, the ties occur with probability zero, and therefore
     \[
        \hat{h}_j = \arg \min_{h} T(G_{hj})
     \]
     is almost surely unique. Because the three head are exchangeable, we have
     \[
        \Pr(\hat{h}_j = h) = \frac{1}{3}, \quad h = 1, 2, 3.
     \]
     Finally, the Gaussian initialization blocks for different feature are independent, therefore $\hat{h}_0, \hat{h}_1, \hat{h}_2$ are independent, which concludes the proof.

\qed
\begin{remark}[Interpretation of Lemma~\ref{lm:early-competition-among-heads}] 
    For each feature $j$, three heads begin from independent small random quartets. Their normalized leading-order dynamics have different escape times, so almost surely one head starts growing before the others. Lemma~\ref{lm:early-competition-among-heads} therefore identifies this early leader. However, it: (1) does NOT yet say that this head permanently wins the feature and (2) it establishes continuous time analysis. Both aspects will be discussed below.
\end{remark}

Lemma~\ref{lm:early-competition-among-heads} establishes the leading head phenomenon for each features but under gradient flow assumption. In the following result, we connect these claims with the actually discrete GD iterates. Essentially, it shows that GD can track auxiliary ODE before escape times of the selected heads, provided the residual remains positive and unbounded. 

\begin{lemma}[Discrete GD follows the auxiliary dynamics before escape] \label{lm:flow-to-discrete-competition-among-heads}
    Fixed one feature and consider three heads $h \in \{1, 2, 3\}$. For each head $h$, let 
    \[
        \theta_h \triangleq (q_h, k_h, v_h, o_h) \in \bbR^4, c(\theta_h) \triangleq o_hv_h\tanh(q_hk_h), 
    \]
    and define the total coefficient $w \triangleq \sum_{h = 1}^3 c(\theta_h)$. Let $\beta > 0$ be the target coefficient and $p > 0$ its loss weight. Consider the square loss $\cL(\theta) = \frac{p}{2} (w - \beta)^2$ and full-batch GD 
    \[
        \theta_h^{n + 1} = \theta_h^n -\eta  \nabla_{\theta_h}\cL(\theta^n).
    \]
    Equivalently, writing $r_n \triangleq \beta - w^n$, and the update is
    \[
        \theta_h^{n + 1} = \theta_h^n + npr_n \nabla c(\theta^n_h). 
    \]
    Initialize $\theta^0_h = \alpha g_h$, for $g_h \in \bbR^4$ and define the normalized parameters $u^n_h \triangleq \frac{\theta^n_h}{\alpha}$. Also define $F_\alpha(u) \triangleq \alpha^{-3} \nabla c(\alpha u)$, and the limiting four-factor vector field 
    \[
        F(u) = \begin{bmatrix}
            u_ku_vu_o\\
            u_qu_vu_o\\
            u_qu_ku_o\\
            u_qu_ku_v
        \end{bmatrix}.
    \]
    Assume that, over the GD iterates under consideration, $0 < r_n \leq \overline{r}$, for some fixed finite $\overline{r}$. Define the residual-weighted auxiliary step and cumulative time by $\delta_n \triangleq \alpha^2 \eta p r_n$, $\zeta_n = \sum_{k < n} \delta_k$. Let $\mathcal{I} \subseteq  \{1,2 ,3 \}$ be any nonempty set of heads and fix $S < \min_{h \in \cI} T(g_h)$. Here, $\cI$ specifies only the heads whose trajectories we approximate, and the residual $r^n$ and the GD updates still involve all three heads, then
    \begin{enumerate}[leftmargin=*]
        \item \textbf{The exact softmax field approaches the four-factor field}: On every compact set $K \subset \bbR^4$, we have $\|F_\alpha - F\|_{C^1(K)} \leq  C_K\alpha^4$ for some constant $C_K$, and 
        \[
            \|H\|_{C^1(K)} \triangleq \sup_{u \in K}\|H(u)\| + \sup_{u \in K}\|DH(u)\|_{\textup{op}}.
        \]
        \item \textbf{Normalized GD follows the auxiliary ODE before escape}: there exists $C_S < \infty$ such that, for all sufficiently small $\alpha$, 
        \[
            \max_{h \in \cI} \|u_h^n - u_h(\zeta_n)\| \leq C_S\alpha^2,
        \]
        for every iterate satisfying $\zeta_n \leq S$. More precisely, the error is $\cO(\alpha^4 + h_\alpha)$, where $h_\alpha \triangleq \alpha^2 \eta p \overline{r}$, and hence is $\cO(\alpha^2)$ for fixed $\eta, p, \overline{r}$. The same error bound also holds at the first crossing index $n_\star = \min \{{n: \zeta_n \ge S \}}$, if it exists, for which $S \leq \zeta_{n_\star} \leq S + h_\alpha$. The step sizes $\delta_n$ may depend on all current model parameters, that is, the approximation requires only their positivity and the bound $\delta_n \leq h_\alpha$, not independence from the iterates.  
    \end{enumerate}
\end{lemma}

\proof 
    We again divide the proof into the following steps. 

    \textbf{Step 1: The exact softmax field approaches the four-factor field:} From $c(q, k, v, o) = ov \tanh (qk)$, we have
    \[\nabla c(q,k,v,o) =
        \begin{pmatrix}
            ovk\,\operatorname{sech}^2(qk)\\
            ovq\,\operatorname{sech}^2(qk)\\
            o\tanh(qk)\\
            v\tanh(qk)
        \end{pmatrix}.
    \]
    Substituting $\theta = \alpha u$, we have $(F_\alpha(u))_q = u_ku_vu_o \operatorname{sech}^2(\alpha^2 u_qu_k)$, and similarly $(F_\alpha(u))_v = u_o \frac{\tanh(\alpha^2u_qu_k)}{\alpha^2}.$ Using Taylor's expansion, for small $z$, we have
    \[ 
        \textup{sech}^2(z) = 1 + \cO(z^2), \quad \tanh z = z + \cO(z^3).    
    \]
    This implies that $\|F_\alpha - F\|_{C^1(K)} \leq C_K \alpha^4$. 

    \textbf{Step 2: Normalized GD follows the auxiliary ODE before escape:} now divide the GD update 
    \[
        \theta_h^{n + 1} = \theta_h^n + npr_n \nabla c(\theta^n_h)
    \]
    by $\alpha$, and since $\nabla c(\alpha u) = \alpha^3 F_\alpha(u)$ by definition, we have
    \[
        u_h^{n + 1} = u_h^n + \alpha^2 \eta p r_n F_\alpha(u^n_h).
    \]
    Here, notice that $ \alpha^2 \eta p r_n = \delta_n$. Therefore, normalized GD is a variable-step Euler discretization of the auxiliary ODE. Since $0 < r_n \leq \overline{r}$, the steps $\delta_n$ satisfies $0 < \delta_n \leq h_\alpha \triangleq \alpha^2 \eta p \overline{r} \rightarrow 0$ for $\alpha \rightarrow 0$. Fix $S < \min_{h \in \cI}T(g_h)$. Since $S$ lies strictly before every escape time, all exact ODE trajectories remain in a common compact set on a slightly larger interval $[0, S + \rho]$. On this set $F$ is Lipschitz with some constant $L$. For one ODE steps, we have
    \[
        u_h(\zeta_{n + 1}) = u_h(\zeta_n) + \delta_n F(u_h(\zeta_n)) + \cO(\delta_n^2). 
    \]
    We then define
    \[
        e_n \triangleq \max_{h \in \cI} \|u_h^n - u_h(\zeta_n)\|.
    \]
    Choose $\rho > 0$ such that $S + \rho < \min_{h \in \cI} T(g_h)$. The selected exact trajectories on $[0, S + \rho]$ lie in the interior of a compact cube $K$, with a positive distance from its boundary. Take $\alpha $ sufficiently small that $h_\alpha < \rho$.

    As long as the numerical iterates lie in $K$, the preceding one-step estimate applies, including to a candidate step that would leave $K$. Since $e_0 = 0$, applying discrete Gronwall inequality gives 
    \[
        e_n \leq C \exp \left(C \sum_{k < n} \delta_k\right) \left(\alpha^4 \sum_{k < n}\delta_k + \sum_{k < n} \delta_k^2\right).
    \]
    For iterates before the threshold and for the first crossing iterate, we have
    \[
        \sum_{k < n}\delta_k = \zeta_n \leq S + h_\alpha, \quad \sum_{k < n} \delta_k^2 \leq h_\alpha \sum_{k < n} \delta_k.
    \]
    Consequently, $e_n \leq C_S(\alpha^4 + h_\alpha)$. For sufficiently small $\alpha$, this bound is smaller than the distance of the exact trajectories from the boundary $K$. Thus a first numerical exit from $K$ is impossible, which implies the estimate throughout the claimed interval, including the first crossing step. If the clock never reaches $S$, the same argument applies to every finite prefix. 

    This argument uses only the realized bounds on the step sizes and their cumulative sum, and it does not require the mesh to be independent of the iterates. Finally, $h_\alpha = \alpha^2 \eta p \overline{r} = \cO(\alpha^2)$ for fixed $\eta, p$, and $\overline{r}$. Therefore $e_n \leq C'_S \alpha^2$ as desired. 
\qed

\begin{remark}
    Lemma~\ref{lm:flow-to-discrete-competition-among-heads} essentially says that as long as no auxiliary trajectory has reached its escape time, the normalized discrete GD trajectory stays close to the continuous four-factor trajectory from Lemma~\ref{lm:early-competition-among-heads}. Therefore, the early head competition predicted by that auxiliary ODE genuinely describes the early stage of the actual GD dynamics. 
\end{remark}

The following lemma shows that, under suitable parameter and sufficiently small step-size, one GD steps preserves the positivity of the residual and does not increase it. In other words, the aggregated coefficient moves toward its target without overshooting. We will combine this result with the parameter bounds provided below to maintain these residual conditions throughout training. 

\begin{lemma}[No residual overshoot] \label{lm:no-residual-overshoot}
    Fixed one feature and write $r = \beta - w$, $s = \eta pr$, and $\theta^+_h = \theta_h + s \nabla c(\theta_h)$, where $c(q, k, v, o) = ov \tanh(qk)$, and $0 < p \leq 1$. Suppose that the weights of every head $\theta_h$ satisfies $\theta_h \in [-B, B]^4$, and that $0 < r \leq D + 1$, where $D = \max(1, t, b, c)$, $R = 12(D + 2)$, $B = 1 + \sqrt{R}$, and $t, b, c$ are defined as in Section~\ref{sec:attention-model-setting}. If $0 < \eta \leq \eta_\star$, where 
    \[
        \eta_\star = \min \left\{\frac{1}{2(D + 1)B^2}, \frac{1}{(D + 1)L}, \frac{1}{36B^6}\right\}, \quad L = 4(2B)^4 + 3(2B)^2 + 1,
    \]
    then every update segment lies in $[-2B, 2B]^4$, and for each head $h$
    \[
        \frac{s}{2}\|\nabla c(\theta_h)\|_2^2 \leq c(\theta^+_h) - c(\theta_h) \leq  \frac{3s}{2} \|\nabla c(\theta_h)\|_2^2.
    \]
    Consequently, $0 \leq w^+ - w \leq \frac{r}{2}$, $\frac{r}{2} \leq r^+ \leq r$, where $w^+ \triangleq \sum_{h = 1}^3 c(\theta^+_h)$ is the aggregate coefficient after one GD update, and $r^+ \triangleq \beta - w^+$. In other words, a positive residual remains positive after one GD step. 
\end{lemma}
\proof
    On a box $[-b_0, b_0]^4$ with $b_0 \geq 1$, we have $\|\nabla c\|_2^2 \leq 4b_0^6$, and $\|\nabla^2 c\|_{op} \leq 4b_0^4 + 3b_0^2 + 1$. Since $s \leq \eta (D + 1)$, the first step-size condition in the definition of $\eta_\star$ implies $s\|\nabla c(\theta_h)\|_\infty \leq B$, so the whole update segment remains in $[-2B, 2B]^4$. On this box, we have $\|\nabla^2 c\|_{op} \leq L$. Using Taylor's expansion, we have
    \[
        c(\theta^+_h) - c(\theta_h) = s \|\nabla c(\theta_h)\|_2^2 + E_h,
    \]
    where the residual $|E_h| \leq \frac{s^2 L }{2}\|\nabla c(\theta_h)\|_2^2$. Since $sL \leq 1$, we have
    \[
        \frac{s}{2}\|\nabla c(\theta_h)\|_2^2 \leq c(\theta^+_h) - c(\theta_h) \leq \frac{3s}{2}\|\nabla c(\theta_h)\|_2^2.
    \]
    Summing over three heads, and note that $\|\nabla c(\theta_h)\|_2^2 \leq 4B^6$, we have
    \[
        0 \leq w^+ - w \leq 18sB^6.
    \]
    Since $s = \eta p r \leq \eta r$, and $\eta \leq (36 B^6)^{-1}$, we have $w^+ - w \leq \frac{r}{2}$. Finally, $r^+ = \beta - w^+ = r - (w^+ - w)$, which implies $\frac{r}{2} \leq r^+ \leq r$ as desired. 
\qed

We now establish an invariant region for the candidate winning head. 

\begin{lemma}[Bounded invariant region for the winning head] \label{lm:bounded-invariant-region-winning}
    Fix one feature and let $h^\star$ be the head with the smallest auxiliary escape time from Lemma~\ref{lm:early-competition-among-heads}. Consider a discrete GD iterate obtained after applying Lemma~\ref{lm:flow-to-discrete-competition-among-heads} at a time sufficiently close to that escape, so that the four parameters of $h^\star$, that is $\theta = (q, k, v, o)$, have positive product and are approximately balanced. After a fixed sign change that leave both $c(\theta)$ and the GD dynamics unchanged, we may assume $q, k, v, o > 0$. Define 
    \[
        Q = q^2 + k^2, V = v^2 + o^2, d = q^2 - k^2, e = v^2 - o^2.
    \]
    Suppose that at the current iterate
    \[
        |d| \leq \frac{Q}{3}, |e| \leq \frac{V}{3}, Q \leq 2V, Q \leq R, V \leq R,
    \]
    and whenever $Q \leq 1$, we have $V \leq 3Q$. Also, assume that Lemma~\ref{lm:no-residual-overshoot} applies to the next GD step, that the next total coefficient satisfies $w^+ < \beta \leq D$, and that the total next-step contribution of the two heads is at least $-1$. Then the updated parameters of $h^\star$ satisfy the same inequalities, that is
    \[
        q^{+2}, k^{+2}, v^{+2}, o^{+2} < R.
    \]
    Therefore, all four updated parameters remain in the box $[-B, B]$.
\end{lemma}
\proof
    Let $x = qk$, $y = vo$. Since $q, k, v, o > 0$, we have $x, y > 0$. The exact GD update gives
    \[
        d^+ = (1 - s^2y^2 \textup{sech}^4x)d, \quad e^+ = (1 - s^2 \tanh^2 x)e. 
    \]
    By the step-size conditions in Lemma~\ref{lm:no-residual-overshoot}, both multipliers lie in $[0, 1]$, and therefore $|d^+| \leq |d|$, $|e^+| \leq |e|$. At the same time, $Q$ and $V$ are non-decreasing, so 
    \[
        |d^+| \leq \frac{Q^+}{3}, |e^+| \leq \frac{V^+}{3}.
    \]
    Thus, the two parameter pairs remain balanced, which implies 
    \[
        qk \geq \frac{Q}{3}, \quad vo \geq \frac{V}{3}.
    \]
    Moreover, the exact updates of $Q$ and $V$ satisfy
    \[
        \Delta Q = 4sxy \textup{sech}^2 x+ s^2 y^2 \textup{sech}^4x Q, \quad \Delta V = 4sy \tanh x + s^2 \tanh^2 xV. 
    \]
    Using the fact that $x \textup{sech}^2x \leq \tanh x$, and together with the balanced bounds and the step-size restrictions, we have $\frac{\Delta Q}{\Delta V} \leq \frac{7}{4} < 2$. Since $Q \leq 2V$, we have
    \[
        Q^+ = Q + \Delta Q \leq 2V + 2\Delta V = 2V^+.
    \]
    The task left is to show that the boundedness. Let $c^+$ denote the contribution of $h^\star$ after the update. Since $w^+ < \beta \leq D$, and the other two heads contribute at least $- 1$, we have $c^+ \leq D + 1$. By balance, we have
    \[
        q^+ k^+ \geq \frac{Q^+}{3}, \quad v^+ o^+ \geq \frac{V^+}{3}.
    \]
    Since $Q^+ \leq 2V^+$, we have $V^+ \geq \frac{Q^+}{2}$. Therefore
    \[
        c^+ = v^+ o^+ \tanh (q^+ k^+) \geq \frac{Q^+}{6}\tanh \left(\frac{Q^+}{3}\right).
    \]
    In the above, the RHS is increasing in $Q^+$, and recall that $R = 12 (D + 2)$ by definition, it exceeds $D + 1$ at $Q^+ = R$. Therefore $Q^+ < R$. 
    
    Now consider $V^+$. If $Q^+ > 1$, we have $q^+ k^+ \geq \frac{1}{3}$, $c^+ \geq \frac{V^+}{3} \tanh \frac{1}{3}$. Since $\tanh \frac{1}{3} > 1/4$, we have $V^+ < 12(D + 1) < R$. If instead $Q^+ \leq 1$, similarly, we can show that $\frac{\Delta V}{\Delta Q} < 3$. Since $V \leq 3Q$, we have
    \[
        V^+ = V + \Delta V \leq 3Q + 3\Delta Q = 3 Q^+. 
    \]
    Therefore 
    \[
        V^+ \leq 3Q^+ \leq 3 < R.
    \]
    Therefore, all defining inequalities are preserved after the update. Finally, 
    \[
        q^{+2}, k^{+2} \leq Q^+ < R, \quad v^{+2}, o^{+2} \leq V^+ < R. 
    \]
    This implies that every updated parameters have magnitude smaller than $\sqrt{R} < B$, which concludes the proof. 
\qed

Finally, we show that the weights of other non-winning heads will stay small throughout the training. 

\begin{lemma}[Remaining residual clock] \label{lm:remaining-residual-clock}
    Fix one feature, and let the designated winning head satisfy the invariant-region conditions of Lemma~\ref{lm:bounded-invariant-region-winning} throughout the interval under consideration. Let $m_n \triangleq \min \{q_n, k_n, v_n, o_n\}$, and $b_\star \triangleq \min \left\{\textup{sech}^2(R/2), \frac{\tanh(R/2)}{R/2}\right\} > 0$. Recall that the GD step is scaled by $s_n = \eta p r_n$. Then we have
    \[
        m_{n + 1} \geq m_n (1 + b_\star s_n m_n^2). 
    \]
    Moreover, $m^{-2}_n - m_{n + 1}^{-2} \geq \frac{3}{4}b_\star s_n$. Therefore, for every $N > n_0$, we have
    \[
        \sum_{n = n_0}^{N  -1}s_n \leq \frac{4}{3b_\star m_{n_0}^2}.
    \]
    In other words, the total remaining residual-weighted clock is finite. 
\end{lemma}
\proof
    Denote $x_n = q_k k_n$. Because the winner remains in the invariant region as in Lemma~\ref{lm:bounded-invariant-region-winning}, we have $0 \leq x_n \leq \frac{Q_n}{2} \leq \frac{R}{2}$. Therefore, by the definition of $b_\star$, we have
    \[
        \textup{sech}^2(x_n) \geq b_\star, \quad \textup{tanh}(x_n) \geq b_\star x_n.
    \]
    Recall that
    \[
        \nabla c(q, k, v, o) = \begin{pmatrix}
            ovk \textup{sech}^2(qk) \\
            ovq \textup{sech}^2(qk) \\
            o\textup{tanh}(qk) \\
            v\textup{tanh}(qk
        \end{pmatrix}.
    \]
    Since every factor is at least $m_n$, each coordinate of the gradient is at least $b_\star m_n^3$. Therefore, after one update, we have
    \[
        m_{n + 1} \geq m_n + b_\star s_n m_n^3 = m_n(1 + b_\star s_n m_n^2). 
    \]
    Now define $z_n \triangleq b_\star s_n m_n^2$. The step-size bounds from Lemma~\ref{lm:no-residual-overshoot} imply $0 \leq z_n \leq 1$, which implies $m^{-2}_{n + 1} \leq m_n^{-2} (1 + z_n)^{-2}$, so 
    \[
        m_n^{-2} - m_{n + 1}^{-2} \geq m_n^{-2} \left(1 - \frac{1}{(1 + z_n)^2}\right) = b_\star s_n \frac{2 + z_n}{(1 + z_n)^2}.
    \]
    For $0 \leq z_n \leq 1$, we have $\frac{2 + z_n}{(1 + z_n)^2} \geq \frac{3}{4}$, which implies $m_n^{-2} - m_{n + 1}^{-2} \geq \frac{3}{4}b_\star s_n$. Telescoping from $n = n_0$ to $N - 1$, we have
    \[
        \frac{3}{4} b_\star \sum_{n = n_0}^{N - 1} s_n \leq m^{-2}_{n_0} - m_N^{-2} \leq m_{n_0}^{-2} \Rightarrow  \sum_{n = n_0}^{N - 1} s_n\leq \frac{4}{3b_\star m_{n_0}^2},
    \]
    as desired.
\qed
\begin{remark}
    Lemma~\ref{lm:remaining-residual-clock} essentially shows that, once one head has established itself as the winner, there is not enough effective training time remaining for the losing heads to catch up. Combined with Lemmas~\ref{lm:early-competition-among-heads} and \ref{lm:flow-to-discrete-competition-among-heads}, together with the stability result of Lemma~\ref{lm:bounded-invariant-region-winning}, this establishes that each feature becomes localized to a particular head. We formalize this localization property in the following result.
\end{remark}

We are now ready to present the main theorem, which establishes that training induced feature localization. 

\begin{customthm}{\ref{thm:training-induce-localization} (restated)}
    For heads $h \in \{1, 2, 3\}$ and features $j \in \{0, 1, 2\}$, denote
    \[
        \theta_{hj} = (q_{hj}, k_{hj}, v_{hj}, o_{hj}), \quad c_{hj} = o_{hj}v_{hj}\tanh(q_{hj}k_{hj}), \quad w_j  = \sum_{h = 1}^3 c_{hj}.
    \]
    Fix $a, b, c > 0$, set $\beta = \left(\frac{a}{2}, b, c\right)$ and $p = \left(1, \frac{1}{2}, \frac{1}{2}\right)$, and train $\cL(\theta) = \frac{1}{2}\sum_{j = 0}^2 p_j (w_j - \beta_j)^2 + \frac{a^2}{8}$ by full-batch GD from the initialization $\theta^0_{hj} = \alpha G_{hj}$, where $G_{hj} \overset{i.i.d.}{\sim} N(0, I_4)$. Then there exists $\eta_\star > 0$, depending only on $a, b, c$, such that the following holds for every fixed step size $0 < \eta < \eta_\star$. For almost every draw of $G$, all sufficiently small $\alpha > 0$, we have:
    \begin{enumerate}[leftmargin=*]
        \item \textbf{Convergence:} the parameter trajectory is bounded and converges to a finite limit with $w^\infty = \beta$. Each residual $r_j^n = \beta_j - w_j^n$ is positive and non-increasing.
        \item \textbf{Localization:} there are labels $h_j = h_j(G)$, independent and uniform on $\{1, 2, 3\}$, and a finite constant $K_G$, independent of $\alpha$, such that 
        \[
            \sup_{n \geq 0}\|\theta^n_{hj}\|_\infty \leq K_G\alpha, \quad \sup_{n \geq 0}|c^n_{hj}| \leq K_G\alpha^4 \textup{ for } h \neq h_j. 
        \]
        Consequently, $c^\infty_{h_jj} = \beta_j + \cO_G(\alpha^4)$.
        \item \textbf{Finite checkpoint:} define $N_\alpha = \lceil \frac{\log(1/\alpha)}{\eta \alpha^2}  \rceil$, for every fixed $M > 0$ and sufficiently small $\alpha$, we have
        \[
            \sup_{n \geq N_\alpha} \max_{j} r_j^n \leq \alpha^M.
        \]
    \end{enumerate}
\end{customthm}

\proof 
    Let $D = \max\{1, a/2, b, c\}, R = 12 (D + 2), B = 1 + \sqrt{R}$, and choose 
    \[
        \eta_\star = \min\left\{\frac{1}{2(D + 1)B^2}, \frac{1}{(D + 1)L}, \frac{1}{36B^6}\right\}, \quad L = 4(2B)^4 + 3(2B)^2 + 1.
    \]
    Fix $0 < \eta \leq \eta_\star$. Because the objective separates across features, we first fix $j$ and suppress its index. Write
    \begin{equation*}
        \begin{aligned}
            &\beta = \beta_j, \quad p = p_j, \quad c(\theta) = ov\tanh(qk), \\
            &w^n = \sum_{h = 1}^3 c(\theta^n_h), \quad r^n = \beta - w^n, \quad s_n = \eta p r^n. 
        \end{aligned}
    \end{equation*}
    The exact GD update is 
    \begin{equation} \label{eq:gd-exact-update}
        \theta_h^{n + 1} = \theta^n_h + s_n \nabla c(\theta^n_h).
    \end{equation}
    Fix a Gaussian draw for which the conclusions of Lemma~\ref{lm:early-competition-among-heads} hold and the three auxiliary escape times are finite and distinct, and note that this event has probability one. Let 
    \[
        h_\star = \arg \min_{h} T(G_h), \quad T_\star = T(G_{h_\star}), \quad T_{(2)} = \min_{h \neq h_\star}T(G_h).
    \]
    We will organize the proof in the following steps.

    \textbf{Step 1: Enter the invariant region}. Choose $S_+$ such that $T_\star < S_+ < T_{(2)}$, and define
    \[
        b_\star = \min \left\{\textup{sech}^2(R / 2), \frac{\tanh(R / 2)}{R / 2}\right\} > 0.
    \]
    Choose $A_0 < \infty $ sufficiently large such that 
    \begin{equation} \label{eq:choosing_A_0}
        \frac{4}{3b_\star A_0^2} < \frac{S_+ - T_\star}{4}.
    \end{equation} 
    By Lemma~\ref{lm:early-competition-among-heads}, the winning auxiliary trajectory eventually has positive product, its factor magnitudes diverge, and their squared ratios converge to one. Therefore, we can choose $S_0 < T_\star$ such that at $S_0$, its smallest factor magnitude exceeds $2A_0$ and its largest-to-smallest squared-factor ratio is less than $\frac{3}{2}$.

    Introduce the normalized parameters and clock 
    \[
        u^n_h = \frac{\theta^n_h}{\alpha}, \quad \zeta_n = \alpha^2 \sum_{k < n}s_k, \quad \delta_n = \alpha^2 s_n.
    \]
    Then Eq.~\ref{eq:gd-exact-update} becomes 
    \begin{equation} \label{eq:u-udpate}
        u_h^{n + 1} = u_h^n + \delta_n F_\alpha(u^n_h), \textup{ where } F_\alpha(u) = \alpha^{-3}\nabla c(\alpha u).
    \end{equation}
    Initially, $w^0 = O_G(\alpha^4)$. Therefore, for sufficiently small $\alpha$, all parameters lie in $[-B, B]^4$ and $0 < r^0 \leq D + 1$. Until the first crossing of $S_0$, a stopping argument using Lemmas~\ref{lm:flow-to-discrete-competition-among-heads}, \ref{lm:no-residual-overshoot} gives us
    \begin{equation} \label{eq:residual-bound}
        0 <  r^{n + 1} \leq r^n \leq D + 1, \quad 0 < \delta_n \leq \alpha^2 \eta p(D + 1), 
    \end{equation}
    and $\theta_h^n = O_G(\alpha)$ for every head. Indeed, Lemma~\ref{lm:no-residual-overshoot} excludes residual failure, while Lemma~\ref{lm:flow-to-discrete-competition-among-heads}, including its candidate-step estimate, excludes a first exit from the parameter box. 

    The clock crosses $S_0$ in finite time. Otherwise, Lemma~\ref{lm:flow-to-discrete-competition-among-heads} would imply $w^n = O_G(\alpha^4)$ for every $n$, and therefore $r^n \geq \beta / 2$ for sufficiently small $\alpha$. This would force $\delta_n \geq \alpha^2 \eta p \beta/2 > 0$, contradicting $\zeta_n < S_0$ for all $n$.

    Let $n_0$ be the first crossing index, Lemma~\ref{lm:flow-to-discrete-competition-among-heads} gives
    \begin{equation}
        \zeta_{n_0} = S_0 + \cO(\alpha^2), \quad \frac{\theta_{h_\star}^{n_0}}{\alpha} \rightarrow u_{h_\star}(S_0).
    \end{equation}
    After a fixed coordinate-wise sign change, we may treat the winning factors as positive. This change preserves $c$ and the GD dynamics because the quartet has positive product. For sufficiently small $\alpha$, we have
    \begin{equation} \label{eq:m_n_0_def}
        m_{n_0} \triangleq \min \{q^{n_0}, k^{n_0}, v^{n_0}, o^{n_0}\} \geq \alpha A_0,
    \end{equation}
    and the squared-factor ratio is at most two.

    Consequently, writing 
    \[
        Q = q^2 + k ^2, \quad V = v^2 + o^2, \quad d = q^2 - k^2, \quad e = v^2 - o^2,
    \]
    we have 
    \[  
        |d| \leq Q / 3, \quad |e| \leq V / 3, \quad Q \leq 2V, \quad V \leq 2Q.
    \]
    Since $Q, V = \cO_G(\alpha^2)$, the winner belongs to the invariant region of Lemma~\ref{lm:bounded-invariant-region-winning}.

    \textbf{Step 2: permanent exclusion of the losing heads.} We prove by induction that, for every $n \geq n_0$, the winner remains in the invariant region $\zeta_n < S_+$, and the losing heads satisfy $\|\theta^n_h\|_\infty \leq K_+ \alpha$ for $h \neq h_\star$ for a constant $K_+$ independent of $n$ and sufficiently small $\alpha$. All parameters will consequently remain in the box required by Lemma~\ref{lm:no-residual-overshoot}. 

    These properties hold at $n_0$. Suppose that they hold through iteration $n$. Lemma~\ref{lm:no-residual-overshoot} gives 
    \begin{equation} \label{eq:res_bound}
        0 < r^{n + 1} \leq r^n, \quad w^{n + 1} < \beta. 
    \end{equation}
    Applying Lemma~\ref{lm:remaining-residual-clock} to the controlled starting states $n_0, \quad, n$ and using Eq.~\ref{eq:m_n_0_def}, we have 
    \[
        \xi_{n + 1} = \zeta_{n_0} + \alpha^2 \sum_{k = n_0}^n s_k \leq \zeta_{n_0} + \frac{4\alpha^2}{3b_\star m^2_{n_0}} \leq S_0 + \cO_G(\alpha^2) + \frac{4}{3b_\star A_0^2} < S_+.
    \]
    Crucially, this bound does NOT assume that the candidate terminal state already belongs to the invariant region. 

    The losing auxiliary trajectories are bounded through $S_+$ since $S_+ < T_{(2)}$. Applying Lemma~\ref{lm:flow-to-discrete-competition-among-heads} to these two heads alone, on the common residual-dependent mesh, gives
    \[
        \|\theta^{n + 1}_h\|_\infty \leq K_+ \alpha, \quad h \neq h_\star.
    \]
    Using the fact that $|\tanh x| \leq |x|$, we have
    \[
        |c(\theta_h^{n + 1})| \leq |o^{n + 1}_h v_{h}^{n + 1}q_h^{n + 1}k_{h}^{n + 1}| \leq K_+^4 \alpha^4.
    \]
    For sufficiently small $\alpha$, the sum of the losing coefficients is therefore at least $- 1$. Together with Eq.~\ref{eq:res_bound}, Lemma~\ref{lm:bounded-invariant-region-winning} implies that the winner remains in its invariant region and returns to the original parameter box. This finishes the induction. 

    Combining these estimates with the pre-entry bounds, and enlarging a constant if necessary, gives
    \[
        \sup_{n \geq 0}\|\theta_h^n\|_\infty \leq K_G \alpha, \quad \sup_{n \geq 0} |c(\theta^n_h)| \leq K_G\alpha^4, \quad h \neq h_\star.
    \]
    Lemma~\ref{lm:no-residual-overshoot} also establishes positivity and monotonicity of the residual through training. 

    \textbf{Step 3: Parameter convergence and exact fitting}. For each sufficiently small $\alpha$, Lemma~\ref{lm:remaining-residual-clock} gives
    \[
        \sum_{n = n_0}^\infty s_n \leq \frac{4}{3b_\star m^2_{n_0}} < \infty.
    \]
    Since $s_n = \eta p r^n$, it follows that 
    \[
        \sum_{n = n_0}^\infty r^n < \infty, \quad r^n \rightarrow 0.
    \]
    All parameter quartets remain in a compact box, so their gradients are uniformly bounded. Eq.~\ref{eq:gd-exact-update}, and the preceding summability bound imply 
    \[
        \sum_{n = n_0}^\infty \|\theta_h^{n + 1} - \theta_h^n\|_\infty < \infty.
    \]
    Therefore, every parameter quartet converges to a finite limit. By continuity and the above, we have $w^\infty = \beta$, and therefore
    \begin{equation}
        c(\theta^\infty_{h_\star}) = \beta - \sum_{h \neq h_\star}c(\theta^\infty_h) = \beta + \cO_G(\alpha^4).
    \end{equation}

    \textbf{Step 4: Finite-checkpoint bound.} Let $\tau \triangleq \min \{n: w^n \geq \beta / 2\}$. This index is finite by convergence. For sufficiently small $\alpha$, it occurs after entry into the invariant region, since $w^{n_0} = \cO_G(\alpha^4)$. From Step 2, we might further choose $\alpha$ sufficiently small that 
    \[
        \sum_{h \neq h_\star} |c(\theta^n_h)| \leq \beta / 4 \quad \textup{ for every $n$}.
    \]
    Since $w^n$ is non-decreasing, we have
    \[
        c(\theta^n_{h_\star}) \geq \beta / 4, \quad n \geq \tau.
    \]
    For the positive winning factors, the invariant-region bounds implies 
    \[
        c = vo\tanh(qk) \leq voqk \leq \frac{VQ}{4} \leq \frac{RQ}{4}.
    \]
    Therefore, the above implies $Q \geq \beta / R$. Moreover
    \[
        o^2 \geq V / 3 \geq Q / 6, \quad qk \geq Q/3,
    \]
    Consequently
    \begin{equation}
        \|\nabla c(\theta^n_{h_\star})\|_2^2 \geq o^2 \tanh^2 (qk) \geq \frac{\beta}{6R}\tanh^2\left(\frac{\beta}{3R}\right) \triangleq m_\beta > 0, \quad n \geq \tau.
    \end{equation}
    Lemma~\ref{lm:no-residual-overshoot} gives
    \[
        w^{n + 1} - w^n \geq \frac{s_n}{2}\|\nabla c (\theta^n_{h_\star})\|_2^2 \geq \frac{\eta pm_\beta}{2}r^n,
    \]
    and therefore
    \begin{equation} \label{eq:residual_ineq}
        r^{n + 1} \leq \left(1 - \frac{\eta p m_\beta}{2}\right)r^n , \quad n \geq \tau. 
    \end{equation}
    The contraction factor is in $[0, 1)$: comparison with the no-overshoot bound $w^{n + 1} - w^n\leq r^n / 2$ gives $\eta p m_\beta / 2 \leq 1/ 2$.

    Before $\tau$, we have $r^n > \beta / 2$. The clock bound implies 
    \[
        S_+ > \zeta_\tau = \alpha^2 \eta p \sum_{n < \tau}r^n \geq \frac{1}{2}\alpha^2 \eta p \beta \tau,
    \]
    which implies $\tau \leq \frac{2S_+}{\eta p \beta \alpha^2}$. It follows that $N_\alpha = \lceil\frac{\log(1/\alpha)}{\eta \alpha^2}\rceil \geq 2\tau$ for sufficiently small $\alpha$. Iterating Eq.~\ref{eq:residual_ineq}, using $r^\tau \leq \beta / 2$, and $1 - x \leq e^{-x}$ gives
    \[
        r^{N_\alpha} \leq r^\tau \exp \left(-\frac{\eta p m_\beta}{2}(N_\alpha - \tau)\right) \leq \beta \exp\left(-\frac{pm_\beta}{4\alpha^2} \log(1/\alpha)\right).
    \]
    For every fixed $M > 0$, the final expression is at most $\alpha^M$ for sufficiently small $\alpha$. Residual monotonicity then gives
    \[
        \sup_{n \geq N_\alpha}r^n \leq \alpha^M.
    \]
    Finally, apply the argument to all three features, taking the minimum of their initialization thresholds and the maximum of their constants. Lemma~\ref{lm:early-competition-among-heads} identifies each winner as 
    \[
        h_j = \arg \min_{h} T(G_{hj}).
    \]
    The escape times are i.i.d. and atom-less across heads, so each winner is unique almost surely and uniformly distributed. The initialization blocks are independent across features, so the winner labels are independent. This concludes the proof.
\qed

\subsection{Additional Results and Omitted Proofs for Section~\ref{sec:localization-enable-beneficial-abalation}}
We now present the detailed proof for our main result, which establish a potential explanation for the emergence of CEL under training dynamics. 

\begin{customthm}{\ref{thm:training-induce-ablation} (restated)}
    Under the assumptions of Theorem~\ref{thm:training-induce-localization}, let $t = a / 2$, and define the event 
    \[
        E = \{h_0 \not \in \{h_1, h_2\}\},
    \]
    where $h_j$ is the head that wins the feature $j \in \{0, 1, 2\}$. Then $\Pr(E) = \frac{4}{9}$. Under such event, set $h_\star = h_0$, and delete this head without changing any other parameters
    \[
        f^n_{-h_\star} = f^n - f^n_{h_\star}.
    \]
    For almost every initialization draw $G$ in $E$, all sufficiently small $\alpha>0$, and every $n\geq N_\alpha$, 
    \begin{equation*}
        \begin{aligned}
            &G_A^n \triangleq R_A(f^n) - R_A(f^n_{-h_\star}) = t^2 + \cO_G(\alpha^4) > 0, \\
            &D_B^n \triangleq R_B(f^n_{-h_\star}) - R_B(f^n) = 3t^2 + \cO_G(\alpha^4) > 0,
        \end{aligned}
    \end{equation*}
    where the error bounds are uniform over $n \geq N_\alpha$. Consequently
    \[
        \liminf_{\alpha \rightarrow 0^+} \Pr_{G}\left(\exists h: R_A(f^{N_\alpha}) - R_A(f^{N_\alpha}_{-h}) \frac{t^2}{2}, R_B(f_{-h}^{N_\alpha}) - R_B(f^{N_\alpha}) \geq \frac{3t^2}{2}\right) \geq 
        \frac{4}{9}.
    \]
\end{customthm}

\proof
    By Theorem~\ref{thm:training-induce-localization}, the head labels $h_0$, $h_1$, $h_2$ are independent and uniform on $\{1, 2, 3\}$. Conditional on $h_0$, each $h_1, h_2$ avoid $h_0$ with probability $\frac{2}{3}$. Therefore, $\Pr(E) = \left(\frac{2}{3}\right)^2 = \frac{4}{9}$.

    Fix a draw $G$ for which Theorem~\ref{thm:training-induce-localization} holds. At an iterate $n \geq N_\alpha$, suppress the superscript $n$, and write 
    \[
        e_j = \beta_j - w_j, \quad z = \sum_{h \neq h_\star} c_{h0}, \quad u_1 = c_{h_\star 1}, \quad u_2 = c_{h_\star, 2}.
    \]
    Theorem~\ref{thm:training-induce-localization}, with $M = 4$, gives 
    \[
        \max_{j}|e_j| \leq \alpha^4.
    \]
    Moreover, every head other than $h_\star$ loses feature $0$, and $h_\star$ loses features $1$ and $2$ on $E$. Hence, for a finite $K_G$, 
    \begin{equation} \label{eq:z}
        |z| \leq 2K_G\alpha^4, \quad |u_1|, |u_2| \leq K_G\alpha^4.
    \end{equation}
    All these bounds hold uniformly for $n \geq N_\alpha$. The intact aggregate coefficients are 
    \[
        w = (t - e_0, b - e_1, c - e_2).
    \]
    After deleting $h_\star$, they become 
    \[
        w_{-h} = (z, b - e_1 - u_1, c - e_2 - u_2). 
    \]
    Using $R_A(w) = w_0^2 + (w_1 - b)^2$, we obtain $G_A = t^2 + \cO_G(\alpha^4)$. Similarly, we can show that $D_B = 3t^2 + \cO_G(\alpha^4)$. The uniform bounds in Eq.~\ref{eq:z} and on $e_j$ make both remainder estimates uniform over all $n \geq N_\alpha$. Because $t > 0$, both quantities are positive for sufficiently small $\alpha$, and the second quantity measures the damage of head ablation on task $B$.

    For the probability conclusion, let $A_\alpha$ denote the event inside the display probability in the statement. The equations above imply that, for almost every $G \in E$, the event $A_\alpha$ holds for all sufficiently small $\alpha$. Thus, along any sequence $\alpha_k \rightarrow 0^+$, we have
    \[
        \liminf_{k \rightarrow \infty} 1_{A_{\alpha_k}}(G) \geq 1_E(G) \quad \textup{almost surely.}
    \]
    Fatou's lemma gives
    \[
        \liminf_{k \rightarrow \infty} \Pr(A_{\alpha_k}) \geq \bbE\left[\liminf_{k \rightarrow \infty}1_{A_{\alpha_k}}\right] \geq \Pr(E) = \frac{4}{9}.
    \]
    Since the sequence was arbitrary, the lower-limit probability follows. 
\qed